\documentclass[]{eptcs}

\usepackage{makeidx}  

\usepackage{amsthm} 

\usepackage{underoverlap} 

\usepackage{latexsym}
\usepackage[pdftex]{graphics}
\usepackage{graphicx}
\usepackage{graphbox}
\usepackage{xypic}

\usepackage{amsmath}
\usepackage{amssymb}
\xyoption{all} 
\usepackage{mathrsfs}

\include{xybtip.mf}

\newcommand{\pns}{P_{ \{ n , s\} }}

\newcommand{\doi}[1]{\textsc{doi}: \href{http://dx.doi.org/#1}{\nolinkurl{#1}}}

\newcommand{\C}{{\mathcal C}}

\newcommand{\sip}{{\sim_\prec}}

\newcounter{thm}

\newtheorem{theorem}[thm]{Theorem}

\newtheorem{lemma}[thm]{Lemma}

\newtheorem{definition}[thm]{Definition}

\newtheorem{proposition}[thm]{Proposition}

\newtheorem{corollary}[thm]{Corollary}

\newtheorem{examples}[thm]{Example}

\newtheorem{remark}[thm]{Remark}

\newtheorem{notation}[thm]{Notation}

\newtheorem{hypothesis}[thm]{Hypothesis}

\def\titlerunning{Information Flow in Natural Languages}
\def\authorrunning{P. Hines}

\begin{document}

\author{Peter M. Hines \institute{YCCSA, University of York}
	\email{peter.hines@york.ac.uk}
	}
 \title{Information Flow in Pregroup Models of Natural Language}

\maketitle
\begin{abstract}
{\em This paper is about pregroup models of natural languages, and how they relate to the explicitly categorical use of pregroups in Compositional Distributional Semantics and Natural Language Processing.  These categorical interpretations make certain assumptions about the nature of natural languages that, when stated formally, may be seen to impose strong restrictions on pregroup grammars for natural languages.   

We formalize this as a hypothesis about the form that pregroup models of natural languages must take, and demonstrate by an artificial language example that these restrictions are not imposed by the pregroup axioms themselves.  We compare and contrast the artificial language examples with natural languages (using Welsh, a language where the `noun' type cannot be taken as primitive, as an illustrative example).

The hypothesis is simply that there must exist a causal connection, or information flow, between the words of a sentence in a language whose purpose is to communicate information. This is not necessarily the case with formal languages that are simply generated by a series of `meaning-free' rules. This imposes restrictions on the types of pregroup grammars that we expect to find in natural languages; we formalize this in algebraic, categorical, and graphical terms. 

We take some preliminary steps in providing conditions that ensure pregroup models satisfy these conjectured properties, and discuss the more general forms this hypothesis may take.
}
\end{abstract}



\section{Introduction}
Lambek pregroups are algebraic structures with a strongly categorical / logical flavor, proposed for modeling linguistic phenomena \cite{JL99}. More recently they have been used heavily within Natural Language Processing -- in particular, the field of Compositional Distributional Semantics \cite{CCS}.  In language processing, their utility is heavily based on a strongly categorical interpretation, and we will freely mix algebraic and categorical descriptions throughout this paper -- hopefully pointing out how the two are related.

A particularly puzzling feature of pregroups is that they seem to be `overspecified'; they have dual notions of {\em expansion} and {\em contraction} (see Section \ref{underover-sect}), and a corresponding neat graphical calculus of {\em underscores} and {\em overscores} (Section \ref{underover-sect}). However, for linguistic purposes, only the expansions / underscores are relevant. From Lambek's original work onwards, this has been something of a mystery; \cite{JL99} proposes that although the `extra structure' has no linguistic interpretation, it nevertheless helps determine the algebraic structures that are in fact useful.

An alternative viewpoint comes from categorical perspectives on Natural Language Processing such as \cite{CCS,GSCCP}. Although this field works with a degenerate notion of pregroups (i.e. compact closed categories -- see Section \ref{compact-sect}), both `underscores' and `overscores' of graphical models are crucial in describing {\em flow of information} or {\em causal connections} in the sense of \cite{KU}. In the linguistic setting, this is a direct claim that they model information flow or interaction  between distinct parts of (grammatically correct) sentences. We refer to this as the {\em categorical hypothesis} (see Section \ref{cat_hyp-sect}). 

By contrast with linguistic models generally, Compositional Distributional Semantics uses a very degenerate form of pregroups. However, the categorical hypothesis is equally applicable to arbitrary pregroup models of natural language. The purpose of this paper is to show that it is not simply a useful heuristic or worldview, but also makes concrete claims about the structure of (pregroup models of) natural languages.  

We first give the basic formalism, including the construction of free pregroups (the particular form used in linguistic models), give a categorical interpretation of this, and relate it to the highly graphical properties of pregroup representations. This is followed by an illustrative example of modeling natural language grammar using pregroups, together with a formal categorical description of this process. We then use this formal description to consider some concrete consequences of the categorical hypothesis. 

We then give a formal language example of a pregroup grammar that does not conform to the predictions of the categorical hypothesis, and explain why we would not expect to find such behavior in a pregroup model of natural language. We conclude by considering potential counterexamples, possible extensions, and take the first steps towards an algebraic axiomatisation of grammars with the predicted behavior.   

\section{Basic definitions \& properties}
In \cite{JL99}, pregroups are defined as follows:
\begin{definition}
\label{pre-def}
A {\bf (Lambek) pregroup} is a monoid $P$ equipped with a partial order $\leq$ compatible with composition (so $p\leq q$ and $r\leq s$ implies $pr\leq qs$),
and two operations $(\ )^l, (\ )^r:P\rightarrow P$ called the {\bf left-} and {\bf right- adjoints} respectively. These 
 are related by the defining identities, $p^lp \leq 1 \leq pp^l$ and $pp^r \leq 1 \leq p^r p$.
 
 Following \cite{WB}, we say that a pregroup $P$ is {\bf proper} if some left- and right- adjoints are distinct -- i.e. there exists some $a\in P$ satisfying $a^r\neq a^l$. Extending this, we say that $P$ is {\bf fully proper} if all non-unit left- and right- adjoints are distinct -- i.e. $a^r\neq a^l$ for all $a\neq 1\in P$.
\end{definition}

\begin{remark}
	Free pregroups are, of course, fully proper. However, there are few concrete models of proper pregroups. The vector space models used in Compositional Distributional Semantics not only identify left- and right- adjoints, but also identify elements with their own adjoints. 
	\end{remark}

The following results on pregroups are standard; see \cite{JL99}
\begin{lemma}
Let $(P,\leq , (\ )^l ,(\ )^r)$ be a pregroup. Then, for all $p,q\in P$, the adjoints are contravariant, and order-inverting, so $(pq)^l=q^lp^l$ and $p\leq q \Leftrightarrow q^l\leq p^l$, 
and similarly for the right adjoint. The identity is its own left and right adjoint, so $1^r=1=1^l$ and the two adjoints are mutually inverse,  so $(p^r)^l =p =  (p^l)^r$.
\end{lemma}

\subsection{Contractions \& expansions} \label{underover-sect}
The following definitions provide useful notational tools for calculations within pregroups: %
\begin{definition}
Given a pregroup $(P,\leq, (\ )^l , (\ )^r)$,  {\bf contractions} are inequalities of the form $uaa^rv \leq uv$ or $u a^lav \leq uv$, 
and {\bf expansions} are inequalities of the form $uv \leq ua^rav$ or $uv \leq uaa^lv$. 
Graphically, contractions (resp. expansions) are indicated by {\bf underscores} (resp. {\bf overscores}). These may be nested and combined, to give a concise notation for expressing inequalities; the decorated word  
\[
\ooalign{ 
  $\overline{n^r \ \overline{n \ n^l} \ \overline{n\ n^l}\ n}$\cr
  $n^r \ n \ \underline{n^l\ n}\ \underline{n^l \ n} $\cr
}
\]
expresses, via the  contractions,  the inequality
$n^r n  n^l  n  n^l  n  \leq  n^r n n^l  n  \leq n^r  n$
and via the expansions, the inequality
$1  \leq   n^r    n    \leq    n^r    n    n^l    n    \leq    n^r    n    n^l   n    n^l    n$. 
\end{definition}

\begin{remark}[Normal forms for expansions and contractions] 
	As demonstrated in \cite{JL99},  {\em any} inequality in a pregroup  may be derived by first considering contractions, and then expansions. This gives a normal form for patterns of underscores \& overscores.  Logically / categorically, this may be seen as a cut-elimination or coherence result.
\end{remark}
 We will rely on normal forms throughout, but first  first consider the categorical status of pregroups and the free pregroup construction.  
 
\section{Free pregroups and quasi-pregroups}\label{free-sect}
Linguistic applications commonly use {\em free pregroups}. These are defined in \cite{WB} in terms of {\em free quasi-pregroups}, as follows:
\begin{definition}
	A {\bf quasi-pregroup} $(Q,\prec, (\ )^l,(\ )^r)$ is a weakening of Definition \ref{pre-def} to the case where $\prec$ is a compatible {\em preorder} rather than a {\em partial order} on the monoid $Q$. 
	The  {\bf free quasi-pregroup $FQP_{G}$} on a poset $(G,\leq_G)$ is defined as follows:
	\begin{itemize}
		\item The underlying monoid is the free monoid $\left( G \times \mathbb Z\right)^*$. 
		\item 
		The left- and right- adjoints are defined by inductively by 
		\begin{itemize}
			\item  $1^l=1=1^r$ 
			\item $(g,z)^r=(g,{z+1})$ and $(g,z)^l=(g,{z-1})$, for all $(g,z)\in G \times \mathbb Z$, 
			\item $(uv)^r=(v^r)(u^r)$ and $(uv)^l=(v^l)(u^l)$, for all $u,v\in \left( G \times \mathbb Z\right)^*$.
		\end{itemize}
		\item
		The preorder $\prec$ is the smallest preorder containing the inductively defined relation $\mathcal R$, given by 
		\begin{itemize}
			\item $g\leq g'\in G$ implies $(g,z) R (g',z)$ for even $z\in \mathbb Z$, and $(g',z) R (g,z)$ for odd $z\in \mathbb Z$.
			\item $(g,z)(g,z+1)\ {\mathcal R}\ 1\ {\mathcal R}\ (g,z+1)(g,z)$ for all $(g,z)\in G \times Z$,
			\item $a \ {\mathcal R}\ b \Rightarrow uav\ {\mathcal R}\ ubv$, for all $a,b,u,v\in (G\times \mathbb Z)^*$.
		\end{itemize}
		By basic algebra, $\prec=\bigcup_{j=0}^\infty {\mathcal R}^j$ is the reflexive transitive closure, or Kleene star, of $R$. 
	\end{itemize}
	The free quasi-pregroup on a set $H$ is given by assuming the discrete partial ordering (i.e. equality).
\end{definition}

\begin{remark} The preorder $\prec$ in the above definition is not a partial order; $(g,0)(g,1)(g,0) \prec (g,0)$ and $(g,0)(g,1)(g,0) \prec (g,0) $, but $(g,0) \neq (g,0)(g,1)(g,0)$. 
\end{remark}

Free pregroups then arise as quotients of free quasi-pregroups by the {\em induced equivalence relation}.

\begin{definition} Let $(P,\prec )$ be a preordered set. The {\bf induced equivalence relation}  $\sip$ is defined by $p \sim_\prec q \ \Leftrightarrow \ p\prec q \ \mbox{ and } \ q \prec p$, and the quotient $P/{\sim_\prec}$ is a poset w.r.t. the {\bf induced partial order} $\prec/\sim_\prec$. 
\end{definition}

\begin{proposition}\label{cong-prop}
	Let $Q,\prec, (\ )^l,(\ )^r)$ be a quasi-pregroup. Then the induced equivalence relation is a congruence on the quasi-pregroup structure.
\end{proposition}
\begin{proof}
	This is a key feature of the construction of free pregroups from free quasi-pregroups found in \cite{WB}, where it is shown that, $ax\sip a'x'$, $(a)^l \sip (a')^l$, and $(a)^r \sip (a')^r$, for all $a\sim_\prec a',x\sim_\prec x'\in Q$.
\end{proof}

\begin{definition}\label{quotient-def}
	The {\bf free pregroup} on a partially ordered set $(G,\leq)$, denoted $P_G$, is defined in \cite{WB} to be the quotient of $FQP_G$ by the induced equivalence relation $\sip$.  
\end{definition}	
The construction of free pregroups from free quasi-pregroups may of course be viewed categorically.
\subsection{Pregroups and quasi-pregroups, categorically}\label{compact-sect}
It is natural to treat pregroups and quasi-pregroups as categories; these are of a special form.
\begin{definition}
	A category is called {\bf posetal} if each hom-set has at most one element. 
	We define $\textsc{posetal}$ to be the category whose objects are posetal categories and whose arrows are functors, and $\textsc{poset}$ to be the full subcategory whose objects are posets, considered as small categories.
\end{definition}

 Interpretation of pregroups as posetal categories is well-established \cite{JL99,CGS}. Any preordered set may be treated as a category with arrows given by the preorder. Monoid composition is then a monoidal tensor, with the interchange law corresponding to compatibility, and the identity being a strict unit. The adjoints are contravariant monoidal endofunctors, and the defining identities become the axioms for a non-commutative form of {\em compact closure}.

Compact closure was originally defined in terms of abstract 2-categorical properties. 
In \cite{KL}, the abstract 2-categorical definition is shown to have a neat characterization in terms of the existence of a duality, and distinguished arrows,giving both a coherence theorem and diagrammatic calculus. We take the following definition, from \cite{KL} as fundamental:
\begin{definition}
	A {\bf compact closed category (CCC)} is 
	a symmetric monoidal category with a dual $(\C,\otimes,\sigma,(\ )^*,I)$, equipped with {\bf unit} \& {\bf co-unit} arrows $\eta_A : I \rightarrow A \otimes A^*$ and $\epsilon_A : A^* \otimes A \rightarrow I$ at every object, $A\in Ob(\C)$. These satisfy the {\bf yanking axiom} $(1_A\otimes \epsilon)(\eta\otimes 1_A)  =1_A  =  (\epsilon_{A^*} \otimes 1_A)(1_A \otimes \eta_{A^*})$.  
\end{definition}

As is well-established \cite{JL99}, the correct setting for pregroups is the non-commutative form of the above:
\begin{definition}
	A {\bf (non-symmetric) compact closed category (NSCCC)} is 
	a monoidal category with left- and right- duals $(\ )^l, (\ )^r$, and left- and right- unit and co-unit arrows $\eta_A^{(l)},\eta_A^{(r)}$ and $\epsilon_A^{(l)},\epsilon_A^{(r)}$ satisfying the obvious non-symmetric analogues of the yanking axiom. 
\end{definition}
Note that the above definitions give CCCs as degenerate cases of NSCCCs. A great deal of literature exists on graphical interpretations and calculii for both symmetric \& non-symmetric  compact closure (e.g. \cite{GTC2,JSV,PS}) so we do not reproduce it here.

\begin{remark} It is well-established that pregroups are posetal non-symmetric compact closed categories \cite{JL99,CCS}. Identical reasoning demonstrates that quasi-pregroups are also posetal non-symmetric compact closed categories.
\end{remark}

\begin{notation}
	[Duals in category theory, the free monoid functor, and the Kleene star] There is an unfortunate clash of notation between the category theorists dual $(\ )^*$, the algebraists Kleene star $(\ )^*$ of relations, and the free monoid construction $(\ )^*$ on sets (which we will treat as a functor). 
	
	Hopefully the intended meaning of the overloaded `star' notation will be clear from the context, as this paper works with fully proper pregroups, and hence distinguishes left- and right- duals.
\end{notation}

\begin{definition} We define $\textsc{nsccc}$ to be the category whose objects are non-symmetric compact closed categories and whose arrows are adjoint-preserving monoidal functors, and $\textsc{ccc}$ to be the full subcategory of compact closed categories. 
\end{definition}

Essentially by definition, free constructions are functorial, and pregroups / quasi-pregroups are no exception --  this follows from the universal property for freeness. We make the following definitions:

\begin{definition}
	We define $FQP:\textsc{poset}\rightarrow \textsc{nsccc}$ to be the functor that takes a poset to the free quasi-pregroup on that poset. Similarly, $FP:\textsc{poset}\rightarrow \textsc{nsccc}$ is the functor that takes a poset to the free pregroup on that poset.
\end{definition}

\begin{proposition}
	There exists a natural transformation from $FQP$ to $FP$ whose components are given by the `quotienting by the induced equivalence relation' step of Definition \ref{quotient-def}. 
\end{proposition}
\begin{proof}
	Categorically, Proposition \ref{cong-prop} states that for a given poset $G$, quotienting by the induced equivalence relation gives a functor of non-symmetric compact closed categories, and hence an arrow in $\textsc{nsccc}(FQP(G),FP(G))$. We then have a family of arrows indexed by the objects of $\textsc{poset}$, and we again appeal to the universal properties implied by freeness to demonstrate that these are indeed the components of a natural transformation.
\end{proof}

\section{Graphical properties of pregroups}
Many of the categorical properties of pregroups are best illustrated graphically,  
using the conventions of overscores and underscores described above. Free pregroups (but not pregroups generally) then satisfy three core properties of {\em undirectedness}, {\em planarity} and {\em acyclicity}. By construction, the underscores and overscores are undirected. The lack of commutativity in the definition corresponds to planarity in an obvious way; in the free setting two distinct underscores (resp. overscores) may not overlap. 
The final property, acyclicity, means that in proper pregroups it is not possible to form closed loops using under- / over- scores. Although intuitively obvious, we provide a proof below to demonstrate how this is closely related to the construction of free pregroups given in Section \ref{free-sect}, and thus how the free pregroup construction relates to the Yanking axiom from the categorical description.
\begin{theorem}\label{acyclic-thm}
	Let $w\in P$ be an arbitrary non-empty word in a fully proper pregroup. Then no pattern of expansions / contractions on $w$ can contain a closed cycle.
\end{theorem}
\begin{proof} 
	We show that the only word satisfying $1\leq w \leq 1$ is the empty word (i.e. the identity); the general result follows by induction.

	Assume some word $w\neq 1\in P$ satisfying $1\leq w \leq 1$. Then there exists a set of expansions demonstrating $ w\leq 1$ and a set of contractions demonstrating $1\leq w$. Thus every symbol in $w$ is part of both a contraction and an expansion. We drawing these as under- / over- scores, and connect the ends of these with their corresponding symbols. This gives a set of (possibly nested)  Jordan curves in the plane:
	
	\begin{figure}[h]
		\caption{Every symbol is part of an expansion \& a contraction}\label{acyclicity}
		\begin{center}
			\includegraphics[scale=0.5]{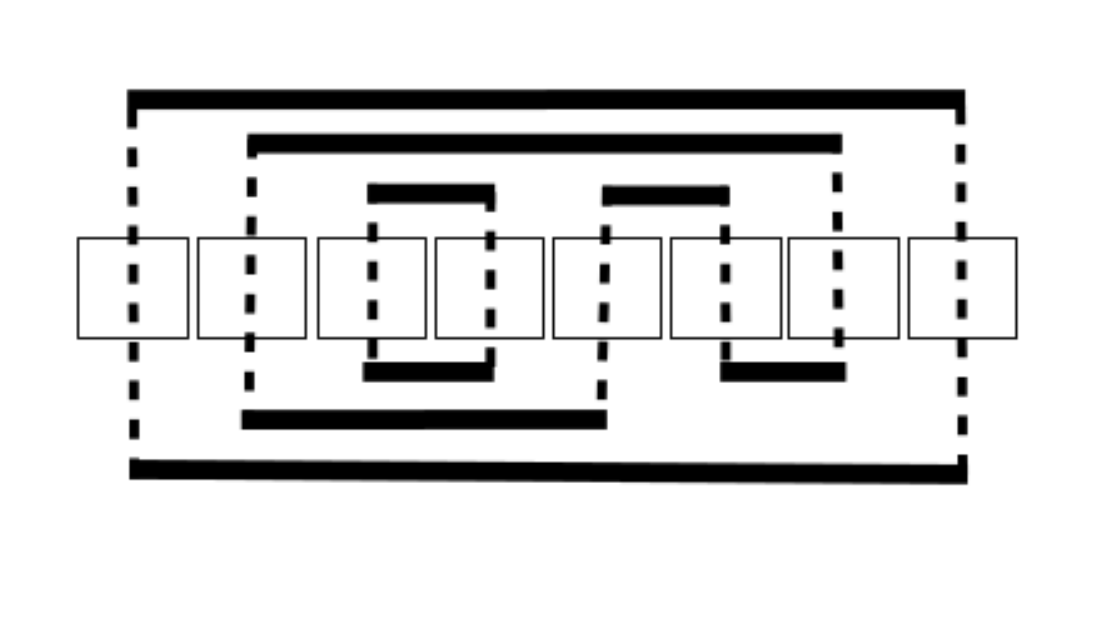}
		\end{center}
	\end{figure}
	
	The symbols of $w$ lie on the intersection of the curve, and a bisecting line -- note that each symbol uniquely determines, and is uniquely determined by, its neighbours {\em on the Jordan curve}. Now choose some symbol $x$ of the word $w$ together with a direction to traverse the curve on which it lies. Trivially, each symbol encountered is some adjoint of $x$; which one is uniquely determined by the direction of movement and whether the current arc lies above or below the bisecting line:
	\[ 
	\xymatrix{
		\ar@{.}[r]	& x^r \ar@{.}[r]	   & x^{\ }  \ar@/_16pt/[l]_{\bf +1} \ar@{.}[r]		& 			&	 \ar@{.}[r]		&	x\ar@{.}[r] \ar@/^16pt/[r]^{\bf -1}  &  x^l \ar@{.}[r]	&       \\ 
		\ar@{.}[r]	& x^l \ar@{.}[r]	   & x  \ar@/^16pt/[l]^{\bf -1} \ar@{.}[r]		& 			&	 \ar@{.}[r]		&	x,\ar@{.}[r] \ar@/_16pt/[r]_{\bf +1}  &  x^r \ar@{.}[r]	& \\ 
	}
	\]
	We label these directed arcs by `weights' as shown, and take the sum around this closed loop.

	Consider a sub-arc of this Jordan curve that starts \& finishes with a contraction, on which the direction of movement (with respect to the bisecting line) does not change,  as shown below:
	\begin{center}
		\includegraphics[scale=0.4]{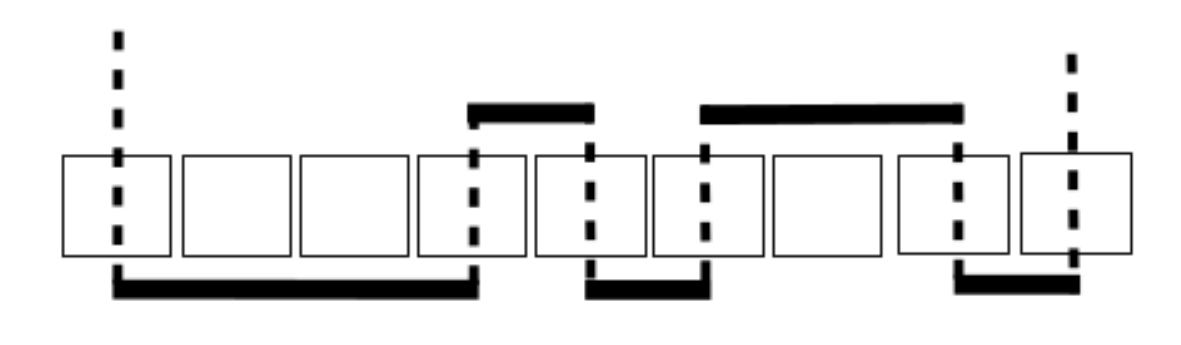}
	\end{center}
	No matter the length of this arc, the sum of weights along it is always $1$ when traversed from left to right, and $-1$ otherwise.  We therefore replace each such arc by a single arc beneath the bisecting line, labeled by either $1$, or $-1$, depending on the direction of traversal, as follows:
	\begin{center}
		\includegraphics[scale=0.4]{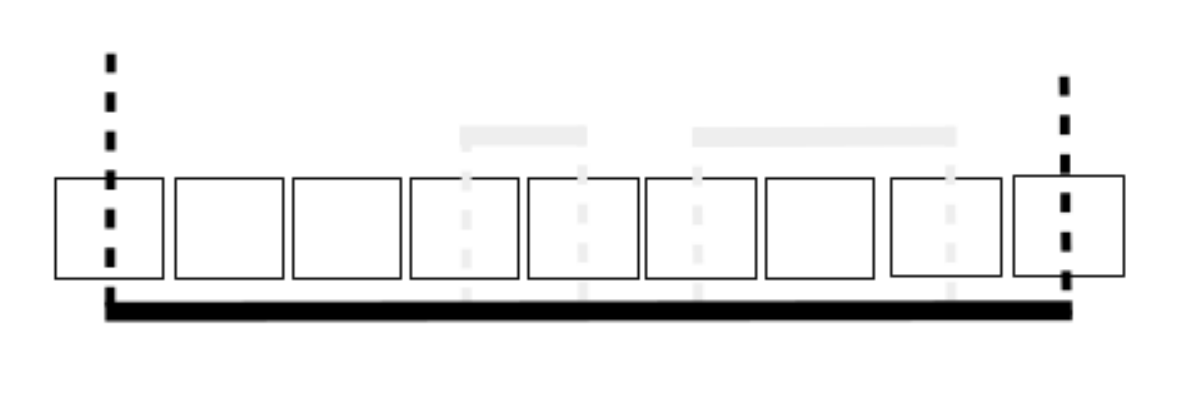}
	\end{center}
	Category theorists will, of course, notice the implicit appeal to the Yanking axioms!

	We do a similar replacement (with a similar interpretation) to repeated unidirectional subarcs that start / finish with an expansion, leaving a closed loop where the direction of movement ( with respect to the bisecting line) changes every time  this line is crossed. 
	
	On the original Jordan curve, a contraction can never be followed by another contraction, nor can an expansion be followed by another expansion, and the same holds in our simplified diagram.  Thus, the total number of contractions in the Jordan curve is equal to the total number of expansions.  
	Therefore, as we change direction at each crossing of the bisecting line, the sum the weights around the entire closed loop is either $2$ or $-2$. 
	This implies the existence of some element whose left and right adjoints are identical, contradicting the assumption that $P$ is a fully proper pregroup.
\end{proof}

\begin{corollary} 
	Let $w$ be a word of length $\geq 2$ in a free pregroup. Then every pattern of expansions / contractions in will leave at least two symbols that are either not part of a contraction, or of an expansion.
\end{corollary}

\begin{corollary}
	Let $Q$ be a (not necessarily free) pregroup. Then acyclicity fails precisely when $q^l=q^r$ for some $q\neq 1\in Q$.
\end{corollary}

\section{Grammatical interpretations}\label{grammar-sect}
Grammatical interpretations are traditionally carried out using free pregroups on sets or posets\footnote{There appears to be little discussion in the literature of the relative merits of using free vs. non-free pregroups within linguistics. We therefore observe that the constructions and results of this paper are equally applicable in the non-free case, but use free pregroups for our concrete examples.}. 
For purely linguistic applications, grammatical types are modeled by words in a free pregroup, and the partial ordering is interpreted as an information ordering; $x\leq y$ expresses that {\em `$x$ is a special case of $y$'}. We are of course interested in pregroup words that are a special case of the `sentence' type, and seek to demonstrate this by a suitable pattern of reductions / underscores.  
This is best illustrated by a concrete example:

\subsection{A worked example}
As a concrete example, we demonstrate how pregroups may be used to describe grammatical structure in a language where the noun phrase type is not primitive, due to mutations of nouns following either possessives or prepositions. We give a pregroup analysis of the grammaticality of the modern Welsh sentence: {\bf ``Dyma fy nghath i''} (See Figure \ref{cymraig} for an explanation of the grammatical constituents).

\noindent
Our generating  poset is 
$\{ n,s,d_{pt},c_1, n_p : n_p \leq n , d_{pt}\leq s\}$, with the following interpretation:

{\begin{center}
		\begin{tabular}{|c|c|}
			\hline
			$s$				& sentence	\\
			$d_{pt}$			& declarative present tense sentence \\
			$n$				& noun phrase \\ 
			$n_p$			& noun (1st person possessive form) \\
			$c_1$			& 1st person confirming pronoun	\\
			\hline
		\end{tabular}
	\end{center}
}
\begin{figure}[h]\caption{An example Welsh sentence}\label{cymraig}
	\begin{center}
		\begin{tabular}{|l|l|c|}
			\hline
			&       declarative p.t. sentence		&  $d_{pt}\leq s$			      \\ 
			&							&										\\
			{\bf Dyma} 	& 	{\bf ``here is''}				&	$d_{pt}n^l$				\\
			&							&										\\
			{\bf fy} 		& 	{\bf ``my''} 				     	&     $n c_1^l n_{p}^l$		   \\
			&							&										\\
			{\bf nghath}	&	{\bf ``cat''} [pos.]				& $n_p\leq n$								\\
			&							&										\\
			{\bf i}			&      [confirming pronoun]			& $c_1$								\\
			&      (first person)				& 									\\
			
			\hline
		\end{tabular}
	\end{center}
\end{figure}

The pattern of contractions 
\[ d_{pt}\underline{ n^l\ n} \underline{ c_1^l \underline{n_{p}^l \ n_p} \ c_1} \]
then demonstrates that $d_{pt}n^lnc_1^l n_{p}^l n_p c_1 \leq d_{pt}\leq s$, and so the given sentence is a special case of a declarative present tense sentence, which is itself a special case of a sentence.

\subsection{Formalizing grammaticality, and the `language bracketing'}\label{bracketing-sect}
The informal description of Section \ref{grammar-sect} may be recast in monoid theoretic and categorical terms.

\begin{definition}
 Assume some set $T$ of grammatical types that contains a distinguished {\em sentence type}, such as 
 \[ T=\{ SENTENCE, TRANSITIVE\_VERB, CONFIRMING\_PRONOUN , \ldots \} \] 
 A {\bf pregroup  model} is simply a function $\mu:T\rightarrow P$ that assigns elements of a pregroup $P$ to each grammatical type. 
When the pregroup is the free pregroup on some generating poset $G$ (as is standard for natural language models), the interpretation of the sentence type is presumed to be a generator, commonly denoted $s\in G$. (We are unaware of any linguistic argument for or against the assumption that the sentence type must be modeled by a generator).
\end{definition}

We formalize this categorically, based on the free monoid functor $(\ )^*:{\bf Set}\rightarrow {\bf Mon}$ from the category of sets to the category of monoids, and the well-known monadicity of the free monoid / underlying set pairing.
\begin{definition} 
	Given arbitrary $M\in Ob({\bf Mon})$, we denote the {\em flattening} associated with the above monad by $(\ )_M^\flat\in {\bf Mon}(M^*,M)$. 
	A pregroup model $\mu:T\rightarrow P$ extends to a monoid homomorphism by the free monoid functor, giving $\mu^*\in {\bf Mon}(T^*,P^*)$. We refer to this as the  {\bf language bracketing} for reasons we explain below. 
	
	It is then usual to work with the composite of the language bracketing and the flattening functor, giving the {\bf pregroup interpretation} $\mu^{\circ}: T^*\rightarrow P$, by the following diagram in $\bf Mon$
	\[ 
	\xymatrix{ T^* \ar[r]^{\mu^*} \ar [dr]|{\mu^{\circ}}				& P^* \ar[d]^{(\ )^\flat} \\
																		& P
	}
	\]  
	Remembering the pregroup structure on $P$, a word of grammatical types $w\in T^*$ is then a grammatically correct sentence iff $\mu^{\circ}(w)\leq \mu^{\circ}(SENTENCE)$.

	(We will, as is traditional, abuse terminology and talk about the pregroup interpretation of the sentence within a language, as well as the pregroup interpretation of a formal string of grammatical types).
	\end{definition}

At least in the free setting, if not generally, the language bracketing may be thought of as taking the pregroup interpretation of a sentence, and bracketing it (implicitly, giving a word in the free monoid over the relevant pregroup). For the interpretation of the sentence, ``Dyma fy ngath i", we have the following:
\[ ( \ d_{pt}n^l \ ) \ ( nc_1^l n_{p}^l ) \ (  \ n_p\ ) \ (  c_1 \ ) \]
This is an element of the free monoid $(P_G)^*$, rather than the free pregroup $P_G$. However, applying the flattening homomorphism will result in a word of $P_G$ that is beneath the sentence type, as required.

\section{Contractions, expansions, and the categorical hypothesis}\label{cat_hyp-sect}
	In the demonstration of Section \ref{cymraig}, only contractions are used to demonstrate that the given sentence is a special case of the sentence type. This is indeed a general principle \cite{JL99}; expansions currently play no role in grammatical applications of pregroups (as opposed to Natural Language Processing applications). From Lambek onwards \cite{JL99}, authors have sought to account for the fact that the formalism has both expansions and contractions.

	The justification given by J. Lambek is that they nevertheless determine the structure of pregroups, even though they play no r\^ole in grammatical applications. A related but stronger viewpoint is implicit or explicit in more recent categorically motivated approaches such as \cite{CCS,CGS,GSCCP}. This claim is that:
	\begin{quotation}
		{\bf \large Both underscores and overscores model the interaction, or flow of information, between components of a (grammatically correct) sentence}
	\end{quotation}
	We refer to this as the {\bf categorical hypothesis}. It is motivated by previous applications of compact closed categories such as {\em logical models} \cite{SA96,PHD}, {\em lambda calculus} \cite{AHS}, {\em Turing machines} \cite{PHD,PH08}, {\em quantum protocols} \cite{AC}, and {\em causal structures} \cite{KU}, where the interpretation of units and co-units as modeling information flow or causal connection is by now well-established. 
	
	Our claim is that this hypothesis is not simply a convenient viewpoint, but makes concrete predictions about pregroup models of natural languages (as opposed to the formal languages that may be expressed in pregroup terms), and thus about grammatical structures for natural language generally.  
	
	\subsection{The importance of information flow}\label{fubar-sect}
	Implicit in the categorical hypothesis is the claim that there is actual interaction, or information flow between all the individual words in a sentence: words are brought together to form a sentence because there is non-trivial interaction between them all.  Although this may seem a triviality, it is not enforced by the pregroup axioms, and we may build example grammars where this is not the case. 
	
	Consider a language with grammatical types $T=\{ SENTENCE,FOO, BAR, DOG, DUCK \}$, together with the free pregroup over the discretely ordered set $G=\{ s , a,b,c \}$, and the pregroup model
	\[ \mu(SENTENCE)=s \ , \ \mu(FOO)=sac^l \ , \ mu(BAR)=ca^r \ , \ \mu(DOG)=a^rb^l \ , \ \mu(DUCK)=  ba^{rr} \] 
	The pregroup interpretation of $FOO.BAR.DOG.DUCK$ is $sac^l ca^r a^rb^l ba^{rr}$, so this identified as a grammatically correct sentence by the following pattern of contractions:  
	\[ s\ \underline{a\ \underline{c^l\ c}\ a^r}\ \underline{a^r\ \underline{b^l\ b}\ a^{rr}} \]
	When considering the overscores, we observe that there are in fact no expansions at all! Thus $FOO.BAR$ and $DOG.DUCK$ are connected by neither underscores nor overscores.  Assuming the correctness of the categorical hypothesis, we conclude that there is no causal interaction / information flow between $FOO.BAR$ and $DOG.DUCK$ in any (grammatically correct) sentence with this typing. We therefore wish to rule it out as an appropriate typing for a meaningful sentence. 
	
	\begin{remark}
	The above objection to this typing as appropriate for a meaningful sentence in a natural language is of course reminiscent of the usual undergraduate objection to the `implication' of boolean logic --- that there is no causal connection between the antecedent and the consequent. Although the use of logical connectives in natural language is not always a good match for their formal logical interpretations, in Section \ref{simply-sect} we consider a restricted fragment that rules out such possibilities, and consider consequences of  
	the categorical hypothesis within this restricted setting. 
	\end{remark}

	\section{Causal connections in sentences and grammars}
	We wished to rule out the $FOO.BAR.DOG.DUCK$ example of Section \ref{fubar-sect} due to a failure of `connectedness', which interprets under the categorical hypothesis as a lack of information flow between the constituents of the sentence. We formalize this as follows
	\begin{definition}
		We assume a set of grammatical types $T$, a pregroup $P$, and a grammatical model $\mu:T\rightarrow P_G$. 		
		Given a word $w\in T^*$, we define its {\bf causal graph} $\mathfrak C_w$ to be the following undirected graph:
		\begin{description}
			\item [Nodes] These are elements of the language bracketing $\mu^*(w)\in P_G^*$.
			\item [Edges] There is an edge between two nodes for each underscore / overscore connecting them, in the image of the language bracketing under the under the `flattening' homomorphism.
		\end{description}
	If $\mathfrak C_w$ is connected, we say that $w$ is a {\bf causally connected word} of $T^*$. If this condition holds for all words of $T^*$ that reduce to a given grammatical type $S\in T$, we say that the grammatical model $\mu:T\rightarrow P_G$ is {\bf $S$-connected}. 
	\end{definition}
	
	\begin{examples}\label{cymraig_con-exam}
	By way of illustration, we contrast the causal graph for our modern Welsh sentence, ``Dyma fy ngath i'' with that of the formal example ``FOO.BAR.DOG.DUCK''. As we are working within free pregroups, we may simply superimpose the language  bracketing with the underscores and overscores. For the Welsh example, we derive:
		\[ {\bf (} d_{pt}\underline{ n^l{\bf )}\ {\bf (}n} \underline{ c_1^l \underline{n_{p}^l  {\bf )} \ {\bf (}n_p{\bf )}} \ {\bf (}c_1{\bf )}}  \]
	Doing the same for the formal example, we derive the following two causal graphs:
			\small 
			\[ 
			\xymatrix{
				&									 & &															& *+[o][F-]{c_1}	\\
				\mbox{\color{black}Dyma fy nghath i}		&		 *+[o][F-]{sn^l}\ar@{-} [rr] & & *+[o][F-]{nc_1^ln_p^l}	\ar @{-}[ur]	 \ar @{-}[dr]		&					\\
				&									 & &															& *+[o][F-]{n_p}
			}
			\]
			
			\[ 
			\xymatrix{
				&		*+[o][F-]{sac^l}\ar@{-}@/^12pt/ [dd]	& &												& &	 \\
				\mbox{\color{black}FOO BAR DOG DUCK}	&												& & *+[o][F-]{a^rb^l}\ar@{-}@/^12pt/ [rr] 		& &	*+[o][F-]{ba^{rr}}\ar@{-}@/^12pt/ [ll]	\\
				&		*+[o][F-]{ca^r}\ar@{-}@/^12pt/ [uu]		& &												& &
			}
			\]
	Although the pattern of underscores \& overscores is almost identical for these two examples, the language bracketing differs significantly. This leads to substantially different causal graphs; one of which we find reasonable for a natural language, and the other we do not. 
	\end{examples}
	
	\subsection{A hypothesis on natural vs formal Languages}
	
	\begin{hypothesis}[The Connectedness Hypothesis]\label{connected-hypo} \ \\
	{\bf \large 	We conjecture, based on the categorical hypothesis, that pregroup models of natural languages are $SENTENCE$-connected.} 
		\end{hypothesis}
		
		\begin{remark} The above hypothesis is not a conjecture in the mathematical / logical sense. It is not amenable to a formal proof but may perhaps be disproved by a convincing counterexample. It shares some common features with the notion of a scientific theory within K. Popper's `critical rationalism', where it is claimed that, {\em ``Every `good' scientific theory is a prohibition: it forbids certain things to happen. The 
			more a theory forbids, the better it is''}  \cite{KP}.  However, it is somewhat exceptional in that it does not seek to relate a man-made model (with predictive power) to natural phenomena; rather, it could be seen as a scientific theory about models. 
		
		This of course, raises the question of what a `convincing counterexample' would look like, and whether a pregroup model of natural language that was not $SENTENCE$-connected would instead be taken as evidence that the model, rather than the hypothesis, was incorrect.
			
		It is also worthwhile to note the underlying assumptions.  It is clearly derived from the categorical hypothesis, and so is underpinned by the assumption that there is indeed `information flow' between the words in a natural language sentence. This in turn assumes that the purpose of natural language grammar is to connect words in a meaningful way, and is not simply a game played with a possibly arbitrary set of rules. 
	\end{remark}
		
Should the above hypothesis be disproved by a convincing counterexample, we would, of course, simply move on to analysing differences between sentences which are,
		and are not, causally connected. However, it is more likely that degeneracies at the level of the models result in
		us seeing spurious causal connectedness – traces of information flow that does not actually take place.
		If we take the FOO-BAR DOG-DUCK example, and identify $a^l = a = a^r$, we see exactly this via the
		resulting pattern of overscores. The linguists convention of using free (and hence fully proper) pregroup models should go some
		way towards eliminating this, but we will observe similar phenomena in `degenerate' models,  where what should be distinct
		grammatical types are mapped to the same pregroup words.

		\section{A simplified class of sentences}\label{simply-sect}
		We now define a simplified fragment of the languages identified as grammatically correct by a pregroup model. The original motivation for this was applications to Compositional Distributional Semantics \cite{CCS}, which aims to combine both meaning and grammar into a single concrete setting. The intention was to rule out grammatical constructs that unavoidably lead to rather structurally complex models. 
		
		\begin{definition} Given a set $G=A\cup \{s\}$,
			we say that a word $w\leq s\in P_G$ {\bf simply reduces to $s$} iff it is of the form $usv$ where $u,v\leq 1\in P_A \leq P_G$. We call the set of all such words the {\bf simply $s$-reducing} words of $P_G$. Given a set $T$ of grammatical types, we say that a pregroup model $\mu:T\rightarrow P_G$ is {\bf simply reducing} iff all elements of $\mu^{\circ}(T^*)$ that reduce to $s$ also simply reduce to $s$.
		\end{definition}
		
		\begin{remark}[The original motivation for simply-reducing words]
			In \cite{PH13}, a strong case is made that `logical' connectives in natural language (such as `and', `or',  etc.) must be polymorphically typed (in a similar manner to \cite{JYG}), \& as a consequence of this any concrete models must be self-dual idempotents of a compact closed category. This would of course imply that they are reflexive \cite{PH99}, and there is a direct link from there to models of pure untyped lambda calculus \cite{AHS}. In terms of computational tractability, this is not where we wish to go! 
			
			Such structures are also not expressible within the vector space models commonly used in distributional semantics; it is only the subcategory of finite-dimensional vector spaces that is compact closed, and this of course contains no non-trivial idempotent objects.
			
			The intention of the above definition was to rule out such grammatical constructs that take as `input' an arbitrary sentence or family of sentences and then `output' a  larger grammatically correct sentence -- the motivation being that unrestricted use of such constructs leads in short order to irreducibly complex models\footnote{A somewhat facetious example being the `Encyclop\ae dia Britannica' problem, where we take any particularly large text and replace every full stop by conjunction. This leads to the question of exactly how large and complex our concrete `sentence type' should be when we allow for naive unrestricted conjunctions.}. 
			Thus, as well as ruling out logical connectives, this also eliminates the possibility of modeling {\em epistemic constructs} such as ``He knows that ...'', ``I believe ...'', and indeed, meta-statements ``It is not provable that ...''. 
		\end{remark}
		
		Logicians -- including the author -- may consider these restrictions to be taking all the fun out of models of meaning. However, they do lead to fragments of language that are amenable to vector space style models, and also provide a secure setting for formalizing some consequences of the categorical hypothesis. We hope to show that they also lead to some interesting non-trivial algebra.

 	
 	\section{$\pns$ --- a toy example}\label{pns-sect}
 	Although deciding whether an individual sentence is causally connected seems to be a simple task (Example \ref{cymraig_con-exam}), deciding whether an entire pregroup model is causally connected appears to be a harder task, even in the simply reducing case.
 	
 	We take some small steps in this direction, within a `toy example' based on the free pregroup over the set $\{ n , s\}$, equipped with the discrete partial order.  This is commonly used as an illustration of pregroup models, with the two basic types correspond to `noun' and `sentence' respectively. 
 	
 	Let us assume a set of grammatical types 
 	\[ T= \{ SENTENCE\ ,\  transVERB \ ,\  intVERB \ ,\ NOUN\ ,\ attADJ\ , \ NONCE  \} \tag{Types\_Example} \]
 	These all, excluding the type $NONCE$, have the obvious intended meaning, so we may go some way towards constructing a simply reducing pregroup model $\mu:T\rightarrow \pns$ by defining 
 	\[   \mu(transVERB)=n^rsn^l \ ,\  \mu(intVERB)=n^rs \ ,\ \mu(attADJ) = nn^r  \tag{Partial\_Model}\]
 	  
 	We now demonstrate how to give a pregroup typing to $NONCE\in T$ that leads to a simply reducing pregroup model $\mu:T\rightarrow \pns$ that is {\em not} $SENTENCE$-connected. We thus exhibit a class of pregroup models that, by Hypothesis \ref{connected-hypo}, we do not expect to model any real-world natural language.

 	\begin{definition}
		In a pregroup $P$, we define the {\bf down-closure of the identity} $\left[ 1_P\right]_\downarrow$ to be the submonoid of words that are below the identity. This contains the identity and closure under composition is ensured by compatibility of composition and partial order (categorically, the interchange law for a monoidal tensor).
	\end{definition}

	\begin{remark}
		From a linguistic point of view, we should be suspicious of any pregroup model where where the interpretation of some type or string of types falls within this monoid; this would correspond to a series of words that could be arbitrarily added to either side of any grammatically correct sentence to give another grammatically correct sentence. 
	\end{remark}
	
			\begin{theorem} Let $T$ and $\mu: T-\{ NONCE\} \rightarrow \pns$ be as defined in Types\_Example and Partial\_Model respectively. Let us also denote the sub-pregroup of $\pns$ generated by $\{ n \} \subseteq \{ n,s \}$ by $P_{ \{ n\} } \leq \pns$. Then for arbitrary $w\in \left[ P_{ \{ n\} } \right]_\downarrow$, completing $\mu$ to a globally defined function by taking $\mu(NONCE)=w$ will give a pregroup typing that is:
				\begin{enumerate}
					\item simply reducing, and 
					\item not $SENTENCE$-connected.
				\end{enumerate} 
				
			\end{theorem}  
			
			\begin{proof}
				It is almost immediate that $\mu$ is indeed simply reducing, as $\mu(NONCE)$ contains no occurrences of $s\in \pns$.
				
				By compatibility of partial order and composition,	
				$\mu^{\circ}(NOUN.intVERB.NONCE)\leq s.1= s$
				so any sentence of this type is also a grammatical sentence. 
				 However, $\mu^{\circ}(NOUN.intVERB.NONCE)$ is of the form $nn^rsw$, where $w\in \left[ P_{ \{ n\} } \right]_\downarrow$. By considering the word bracketing, there is no causal connection between $(n^rs)$ and $\mu^{\circ}(NONCE)$. Thus $\mu:T\rightarrow \pns$ is not causally connected.
				 \end{proof}

			\begin{remark}	
				Mathematically, the above toy example is based on elements within the down-closure of the identity in the free pregroup generated by a singleton. Such elements are amenable to a simple  graphical description, based on the proof of Theorem \ref{acyclic-thm} -- the more general case appears to be significantly harder.
			\end{remark}

	\section{Comments on the assumption of freeness} 
	Linguistically, it is standard to base grammatical models on free pregroups. Even in areas where non-proper pregroups are used (such as natural language processing and compositional distributional semantics), the grammatical typing is commonly thought of as a mapping into a free pregroup, followed by some appropriate quotient. 
	
	Due to this linguistic convention, much of this paper has used free pregroups for illustrative purposes; however, concepts such as {\em language bracketing}, {\em causal connectedness} and {\em the categorical} and {\em connectedness hypotheses} are defined generally, rather than in simply in the free pregroup setting (notably, as observed in Section \ref{bracketing-sect}, the language bracketing is significantly more subtle in the non-free case). 

 \section{Conclusions \& future directions}
 This paper is of course very preliminary work. Our overall thesis is that the categorical hypothesis of Compositional Distributional Semantics makes predictions about the forms of grammar we do and do not expect to find in natural languages. We have attempted to formalize this to the point where concrete predictions can be made and potentially tested. 
 
 We do still require a better algebraic and categorical understanding of the implications of these restrictions, including characterizing the grammars that do and do not satisfy these conditions.

 \subsection{Other forms of connectedness, and the definition of types} The motivation for the claim that natural languages should be sentence-connected is clear, and it is natural to wonder whether the same applies to other grammatical types. 
 Given a series of natural language words that come together to make up something of, say, the $NOUN\_PHRASE$ type, the same intuition would suggest that pregroup models of natural language should be $NOUN\_PHRASE$ connected. 
 
 Based on this intuition, it is then hard to conceive of a grammatical type $U$ where we would {\em not} expect $U$-connectedness of pregroup models (with the possible exception of the constructions outlined in Section \ref{simply-sect}). We could then consider a much stronger form of Hypothesis \ref{connected-hypo}, and speculate that pregroup models of natural language should be $T$-connected, for all grammatical types $T$. 

 This raises the question of what should be considered as a `grammatical type'? The discussion of Section \ref{bracketing-sect} presents them as an a priori given, but it is worthwhile to consider their origins. In a language with $VERB-SUBJECT-OBJECT$ ordering, we are happy to call $VERB$ or $VERB\_PHRASE$ a grammatical type, but do not refer to the $SUBJECT-OBJECT$ pair of noun-phrases as a single type. Of course, a pair of $NOUN\_PHRASE$ types will not in general be connected, but will rely on the verb phrase to establish connections between them in a complete sentence. We may conjecture that causal connection or information flow is what distinguishes types from more arbitrary collections of words.

\section*{Acknowledgements} 
I am grateful to the anonymous referees of CAPNS for perceptive comments highlighting where clarifications or revisions were needed. I am also grateful for many discussions on linguistics, NLP, algebra, and category theory with the usual suspects, including Steve Clarke, Bob Coecke, Chris Heunen, Suresh Manandhar, Mehrnoosh Sadrzadeh, and Phil Scott.


\bibliographystyle{eptcs.bst}
\bibliography{pg_bib}

\end{document}